\documentclass[12pt]{article}
\title{On an Edge-Preserving Variational Model for Optical Flow Estimation}
\author{Hirak Doshi\footnote{Corresponding Author}\: and N. Uday Kiran\\Department of Mathematics and Computer Science\\Sri Sathya Sai Institute of Higher Learning, Andhra Pradesh, India\\Email: \{hirakdoshi,\:nudaykiran\}@sssihl.edu.in}
\date{}
\usepackage{mathtools}
\usepackage{adjustbox}
\usepackage{amsmath,amsfonts,amssymb,amsthm}
\usepackage{relsize}
\usepackage{xcolor}
\usepackage{subfig}
\usepackage{float}
\usepackage{multirow}
\usepackage{tikz}
\usepackage{algorithm}
\usepackage[noend]{algpseudocode}

\usepackage[a4paper,left=3cm,top=3cm,right=3cm,bottom=3cm]{geometry}
\usepackage[utf8]{inputenc}

\usepackage{array}
\newcolumntype{L}{>{\centering\arraybackslash}m{11cm}}

\usepackage[T1]{fontenc}
\usepackage{tikz}
\usepackage{setspace}
\usetikzlibrary{shapes.arrows}
\usetikzlibrary{arrows,decorations.pathreplacing,positioning}
\theoremstyle{definition} 

\newtheorem{lemma}{Lemma}

\newtheorem{theorem}{Theorem}

\newtheorem*{keywords}{Keywords}

\newtheorem*{class}{Mathematics Subject Classification 2020}
\DeclareMathOperator{\divr}{div}

\DeclareMathOperator{\med}{med}

\DeclareMathOperator*{\argmin}{arg\,min}
\DeclareMathOperator*{\proj}{\textbf{proj}}
\DeclareMathOperator*{\argmax}{arg\,max}
\usepackage{hyperref}
\hypersetup{
	colorlinks=true,
	linkcolor=blue,
	citecolor=red,      
	urlcolor=cyan,
}
\allowdisplaybreaks
\begin{document}
	\onehalfspacing
	\fontsize{13pt}{13pt}\selectfont
	\maketitle
	\begin{abstract}
	It is well known that classical formulations resembling the Horn and Schunck model are still largely competitive due to the modern implementation practices. In most cases, these models outperform many modern flow estimation methods. In view of this, we propose an effective implementation design for an edge-preserving $L^1$ regularization approach to optical flow. The mathematical well-posedness of our proposed model is studied in the space of functions of bounded variations $BV(\Omega,\mathbb{R}^2)$. The implementation scheme is designed in multiple steps. The flow field is computed using the robust Chambolle-Pock primal-dual algorithm. Motivated by the recent studies of Castro and Donoho we extend the heuristic of iterated median filtering to our flow estimation. Further, to refine the flow edges we use the weighted median filter established by Li and Osher as a post-processing step. Our experiments on the middlebury dataset show that the proposed method achieves the best average angular and end-point errors compared to some of the state-of-the-art Horn and Schunck based variational methods. 
	\end{abstract}
	\begin{keywords}
		Optical flow, Edge-preserving, Iterated median filtering, Primal-dual scheme.
	\end{keywords}
	\begin{class}
	35A15, 35J47, 35Q68.
	\end{class}
\section{Introduction}
\label{sec:1}
The seminal work of Horn and Schunck (HS) \cite{Horn} opened up newer directions in optical flow estimation. Many modern methods derive motivation from the original HS formulation. These methods combine a data fidelity term based on some constancy assumptions along with some spatial and temporal flow regularization terms. Based on these terms, an objective functional is defined and optimized.

Modifications based on the data term has been an active area of research. A pioneering work in this regard is due to Aubert et.al. \cite{Aubert} who proposed to penalize the $L^1$ norm of the Optical Flow Constraint (OFC) instead of the standard $L^2$ penalization. Some other notable choices for the data term is the Charbonnier \cite{bruhn2} and the Lorentzian \cite{black}.

The classical quadratic regularization of the HS model gives smooth vector fields. Due to this isotropic behaviour important edge information is lost. An early modification involving the total-variation regularization was proposed by Cohen \cite{Cohen}. Several improvements have been proposed subsequently in the notable works of \cite{Aubert,martin2,wedel}. Hinterberger et. al. \cite{Hinterberger} analyzed several of these models within the framework of calculus of variations.

The improvements in these HS based formulations can be seen with the gain in accuracy of the flow. In the last decade, modern optimization practices have significantly improved the accuracy as seen from the results on the middlebury dataset \cite{sun2,zhang}. With an effective implementation strategy combining these modern principles, Sun et. al. \cite{sun1, sun2} highlight that the classical HS based methods perform better than some of the modern methods and therefore remain largely competitive.

An important step in improving flow accuracy is the application of median filter at every warping step to remove outliers in the flow. However there are some known drawbacks of the standard median filter. The first drawback is it leads to higher energy solutions. Secondly, Castro and Donoho \cite{donoho} showed that median and linear filters have the same asymptotic worst-case mean squared error (MSE) when the signal-to-noise ratio is of the order 1. Within a decision-theoretic framework, they proposed the heuristic of iterated median filtering. In this, median filtering is first applied at a fine scale followed by a coarse scale median filter. At the coarser level, the algorithm exploits the nonlinearity of the median while at the finer level it increases the SNR. The first issue was handled by the authors in \cite{morel, sun1, sun2} who modified the original objective function by introducing a weighted \emph{non-local} term. While the \emph{non-local} term governs a smoothness assumption within a specified region, the weight term captures the likelihood of certain pixels belonging to the same surface. A formal connection between the \emph{non-local} term and the weighted median filtering was established by Li and Osher \cite{li}.

Our main goal is to design an effective implementation of the edge-preserving $L^1$ regularization approach to optical flow. We propose a new variational model which penalizes the optical flow constraint in $L^1$ sense combined with the total-variation regularization and an additional constraint penalizing the divergence of the flow with an anisotropic weight term. The mathematical study of the proposed model is done in $BV(\Omega,\mathbb{R}^2)$.

The implementation of our model is designed in mutiple steps. Since the formulation involves a non-smooth optimization, we incorporate the primal-dual Chambolle-Pock algorithm \cite{cham} which leads to an efficient convergence rate of the order $\mathcal{O}(1/N)$. The scheme is combined with the modern implementation practices. Keeping in mind the drawbacks of standard median filter discussed above, we extend the heuristic of iterated median filtering in our estimation. To further refine the flow, we use the heuristic of weighted median filtering by Li and Osher \cite{li} as a post processing step. Our experiments on the middlebury dataset show that the proposed method achieves the best average AAE and EPE compared to some of the state-of-the-art HS based variational methods.

The organization of the paper is as follows. In Section \ref{sec:2} we propose our variational model and study the mathematical well-posedness in Section \ref{sec:3}. In Section \ref{sec:4} an effective numerical scheme is developed using the Chambolle-Pock primal-dual algorithm and discussed. We then discuss the modren implementation practices in details and the choice of different filering techniques used in our work in Section \ref{sec:5}. Subsequently, in Section \ref{sec:6}, the details of the experiments are presented and the results on the middlebury dataset are shown.

\section{Variational Formulation}
\label{sec:2}
Our general variational formulation is:
\begin{equation}
	E_0(\textbf{u}) = \int_\Omega \rho_d(|f_t+\nabla f\cdot \textbf{u}|) + \gamma\sum_{i=1}^2\int_\Omega \rho_r(|\nabla u_i|)+\eta\int_\Omega \phi(|\nabla f|)(\nabla \cdot \textbf{u})^2,
	\label{fun1}
\end{equation}
where $\textbf{u}=(u_1,u_2)$ is the optical flow, $\gamma,\eta$ are regularization parameters. The function $\rho_d:\mathbb{R}\to\mathbb{R}$ is the function of the OFC and $\rho_r:\mathbb{R}\to\mathbb{R}$ is the regularization term which governs how the flow spatially varies across the image. The function $\phi:\mathbb{R}\to\mathbb{R}$ is a non-negative, monotone decreasing weight function. Also $\phi\in C^\infty(\Omega)$.
Some of the popular choices for the data term is listed in the following table:
\begin{table}[H]
	\centering
	\begin{tabular}{|c|c|}
		\hline  
		& $\rho_d(x)$ \\[0.3cm]
		\hline &\\[-1.5ex]
		Quadratic\cite{Horn} & $x^2$ \\[0.3cm]
		$L^1$\cite{wedel} & $|x|$\\[0.3cm]
		Charbonnier\cite{bruhn2} & $\sqrt{x^2+\epsilon^2},\epsilon >0$\\[0.3cm]
		Lorentzian\cite{black} & $\log\Big(1+\frac{x^2}{2\sigma^2}\Big)$\\[0.3cm]
		\hline
	\end{tabular}
	\caption{Some choices for the data term.}
\end{table}
The quadratic penalization of the data term is the simplest choice started with the approach of Horn and Schunck \cite{Horn}. Black and Anandan \cite{black} highlighted the difficulties involved with least-squares estimation in the case of multiple image motion. They proposed a robust estimation framework with the Lorentzian data term. An influence function proportional to the derivative of the robust data term was used an indicator for minimizing outliers in the flow. Aubert et. al.\cite{Aubert} proposed a discontinuity-preserving variational optical flow model involving the $L^1$ norm of the OFC. To explore the important advantages of the local and global methods for optical flow estimation, Bruhn et. al. \cite{bruhn2} combined both these techniques. They proposed a hybrid model that gives a better understanding of the flow structures, yields dense flow fields which are robust to noise.

\subsection{The Choice of Regularization} 
The obvious drawbacks of the quadratic regularization motivated researchers to look for more robust alternatives. Nagel \cite{nagel}, Nagel and Enkelmann \cite{nagelenk} considered an oriented smoothness terms to handle occlusions in the flow. Cohen \cite{Cohen} used the total-variation regularization. The total variation measures small oscillations without penalizing the discontinuities in the flow, see \cite{weickert}. Aubert et. al. \cite{Aubert} highlighted that a robust edge-preserving funtion $\rho_r$ must have some desired properties:
\begin{enumerate}
	\item $\rho_r'(0)=0$. This condition ensures that the diffusion process is isotropic within the homogeneous regions.
	\item $\displaystyle\lim_{x\to 0}\dfrac{\rho_r(x)}{x}>0$. This condition implies that $\rho_r$ has a linear growth which helps in establishing the coercivity property.
	\item $\displaystyle \lim_{x\to 0}\rho_r''(x)>0$. This condition shows that the functional is strictly convex and hence has a unique minimum.
	\item $\lim_{x\to\infty}\dfrac{\rho_r'(x)}{x}=0,\quad \lim_{x\to\infty} \rho_r''(x)=0$. These conditions further imply that the rate of diffusion decreases rapidly along the edge information in the flow.
\end{enumerate}

\subsection{Additional Constraint}
An additional regularizing term is introduced in the functional (\ref{fun1}) which penalizes the divergence of the flow coupled with a weight term $\phi$, an edge-stopping function. Two popular choices for $\phi$ were highlighted by Perona and Malik \cite{perona},
\begin{equation*}
	\phi(|\nabla f|)=\frac{K^2}{K^2+|\nabla f|^2}, \quad \phi(|\nabla f|)=\exp(-|\nabla f|^2/K^2).
\end{equation*}
where $K$ is a threshold parameter. The scale-space generated by these two functions showed that the first choice leads to edges over wider regions while the second choice favours edges with high contrast. Sand et. al. \cite{sand} used the flow divergence and brightness constancy error to detect occlusions in the flow. This is directly accounted for in our functional.

Based on the above discussions, we look for a velocity field $\textbf{u}$ which minimizes the energy of the functional:
\begin{equation}
E_0(\textbf{u}) = \int_\Omega |f_t+\nabla f\cdot \textbf{u}| + \gamma\sum_{i=1}^2\int_\Omega |\nabla u_i|+\eta\int_\Omega \frac{K^2}{K^2+|\nabla f|^2}(\nabla \cdot \textbf{u})^2.
\label{fun2}
\end{equation}	
\section{Well-posedness}
\label{sec:3}
\subsection{Preliminaries}
Let $\Omega\subset\mathbb{R}^2$ be an open and bounded set. The space of functions of bounded variation is defined as
\begin{equation*}
	BV(\Omega)=\Big\{u\in L^1(\Omega):|Du|(\Omega)<\infty\Big\},
\end{equation*} 
where $Du$ is the gradient of $u$ in the sense of distributions, $|Du|(\Omega)$ is total variation of $Du$ given by
\begin{equation*}
	|Du|(\Omega):=\int_\Omega |Du| = \sup\Big\{\int_\Omega u\divr \varphi:\varphi\in C^1_c(\Omega)^2, |\varphi|_\infty\le 1\Big\}.
	\end{equation*}
The space $C^1_c(\Omega)$ is a space of continuously-differentiable functions with compact support, $|\varphi|_\infty=\sup\Big\{\sum_{i=1}^2\varphi_i^2\Big\}^{1/2}$. When $u\in BV(\Omega)$, $Du$ is identified as a vector-valued Radon measure. This is a direct consequence of the Radon-Nikodym theorem and Reisz Representation theorem, (see chapter 2, \cite{Aubert2}). The space of Radon measures is denoted by $\mathcal{M}(\Omega)$. The strong, weak, weak-$^*$ convergence in $BV(\Omega)$ is denoted by $\rightarrow,\rightharpoonup,\overset{\ast}{\rightharpoonup}$ respectively. We say that $u_n\overset{\ast}{\rightharpoonup} u$ in $BV(\Omega)$, if $u_n\rightarrow u$ in $L^1(\Omega)$ and $Du_n\rightharpoonup Du$ in $\mathcal{M}(\Omega)$, i.e. for all $\varphi\in C_c(\Omega)^2, \int_\Omega \varphi Du_n\rightarrow \int_\Omega \varphi Du$.
\subsection{Main Result}
Let us consider our proposed functional:
	\begin{equation}
	E(\textbf{u})=\int_\Omega |f_t+\nabla f\cdot \textbf{u}| + \gamma\sum_{i=1}^2\int_\Omega \psi(D u_i)+\eta\int_\Omega \phi(|\nabla f|)(\nabla \cdot \textbf{u})^2.
	\label{fun3}
\end{equation}
The notation $\displaystyle \int_\Omega \psi(Du_i)$ is formal here since we are studying the minimization problem in the space $BV(\Omega,\mathbb{R}^2)$. By Lebesgue Decomposition theorem, we have
	\begin{equation*}
	Du = \nabla u \cdot \mathcal{L}^2 + D^s u,
\end{equation*}
where $\nabla u<< \mathcal{L}^2$, i.e. $\nabla u$ is the absolutely continuous part with respect to the Lebesgue measure $\mathcal{L}^2$ on $\mathbb{R}^2$ and $D^s u$ is the singular part. Aubert et.al. \cite{Aubert} considered an integral representation of the second term in (\ref{fun3}) by the following:
\begin{equation}
	\int_\Omega \psi(Du) =\int_\Omega \psi(|\nabla u|) dx + \psi^\infty(1)\int_\Omega D^s u,
	\label{fun4}
\end{equation}
where $\psi^\infty$ is the asymptote function given as $\psi^\infty(t)=\lim_{s\to 0}\frac{\psi(ts)}{s}$. The relaxed functional (\ref{fun4}) is lower semi-continuous in $BV-w^*$ topology. The next step in the study of the functional (\ref{fun3}) is the interpretation of the third term
\begin{equation*}
\int_\Omega \phi(|\nabla f|)(\nabla \cdot \textbf{u})^2
\end{equation*}
as a Radon measure. To make a sense of this, we rely upon the concept of divergence-measure fields studied in \cite{chen2}. We define
\begin{equation*}
	|\nabla\cdot \textbf{u}|(\Omega):=\int_\Omega |\nabla\cdot \textbf{u}|=\sup\Big\{\int_\Omega \textbf{u} \cdot \nabla\varphi:\varphi\in C^1_c(\Omega), |\varphi|_\infty\le 1\Big\}.
\end{equation*}
The term $|\nabla\cdot \textbf{u}|(\Omega)$ is the total variation of $\nabla\cdot \textbf{u}$. To show the lower-semicontinuity of $(\ref{fun3})$ we need to show the lower semiconitinuity of the third term in $BV-w^*$ topology.
\begin{lemma}
	The term
	\begin{equation*}
		\int_\Omega \phi(|\nabla f|)(\nabla \cdot \textbf{u})^2
	\end{equation*}
is lower semi-continuous in $BV-w^*$ topology.
\end{lemma}
\begin{proof}
	Let us consider a sequence $\textbf{u}_n\to\textbf{u}$, $\varphi\in C^1_c(\Omega), |\varphi|_\infty\le 1$. Then
	\begin{equation*}
		\int_\Omega \textbf{u} \cdot \nabla\varphi = \lim_{n\to\infty} \int_\Omega \textbf{u}_n \cdot \nabla\varphi \le \liminf_{n\to\infty}  |\nabla\cdot \textbf{u}_n|(\Omega).
	\end{equation*}
Therefore $\displaystyle |\nabla\cdot\textbf{u}|(\Omega)\le \liminf_{n\to\infty}  |\nabla\cdot \textbf{u}_n|(\Omega)$. Consequently,
\begin{equation*}
	\int_\Omega \phi(|\nabla f|)(\nabla \cdot \textbf{u})^2\le \|\phi\|_{L^\infty}\int_\Omega |\nabla \cdot \textbf{u}|^2\le \|\phi\|_{L^\infty}\int_\Omega 
	|\nabla \cdot \textbf{u}|\le \|\phi\|_{L^\infty} \liminf_{n\to\infty} |\nabla\cdot \textbf{u}_n|(\Omega).\qedhere
\end{equation*}
\end{proof}
\begin{theorem}
Let $f\in W^{1,\infty}(\Omega)$, i.e. f is Lipschitz continuous and
\begin{enumerate}
	\item $\psi:\mathbb{R}\to\mathbb{R}$ is a strictly convex, non-decreasing function.
	\item There exists $a>0,b\ge 0$ such that $ax-b\le\psi(x)\le ax+b$, i.e. $\psi$ has a linear growth property.
\end{enumerate}
Then there exists a unique minimum of the functional (\ref{fun3}) in $BV(\Omega,\mathbb{R}^2)$.
\end{theorem}
\begin{proof}
	The functional (\ref{fun3}) is well-defined in $BV(\Omega,\mathbb{R}^2)$ due to the inclusion of $BV(\Omega,\mathbb{R}^2)$ in $L^2(\Omega)^2$. Assumption 2 in the hypothesis ensures that the functional is $BV$-coercive. Hence for any bounded sequence there is a weak convergent subsequence in $BV-w^*$. From Lemma 1 and previous discussions, we conclude that the functional (\ref{fun3}) is lower-semicontinuous in $BV-w^*$ topology. This shows the existence.
\end{proof}
\section{The Primal-Dual Framework}
\label{sec:4}
	\subsection{Preliminaries}
	Let $\Omega\subset \mathbb{R}^2$ be an open, bounded set, $\mathcal{X}, \mathcal{Y}$ be two finite-dimensional vector spaces with the scalar product $(\cdot,\cdot)$ and the norm $\|\cdot\|$. Denote the primal variable $\textbf{u}=(u_1,u_2)$ and the dual variable $\textbf{d}=(d_1,d_2,d_3)$. We first consider the variational problem in the following form
	\begin{equation}
		\argmin_{\textbf{u}} G(\textbf{u}) + F(K\textbf{u})\,.
		\label{prim}
	\end{equation}
	where $F,G:\mathcal{X}\to [0,\infty]$ are convex, proper and lower-semicontinuous functionals, $K:\mathcal{X}\to \mathcal{Y}$ is a continuous, linear operator. The equivalent primal-dual formulation is given as
	\begin{equation}
		\argmin_{\textbf{u}}\argmax_{\textbf{d}}\: G(\textbf{u}) + (K\textbf{u},\textbf{d}) - F^*(\textbf{d})\,
		\label{pd}
	\end{equation}
	where $F^*$ is the convex conjugate of $F$. Given a $\tau,\sigma >0$, an initial $(\textbf{u}^0,\textbf{d}^0)\in \mathcal{X}\times \mathcal{Y}$, the Chambolle-Pock algorithm solves the saddle point problem (\ref{pd}) by the following algorithm:
	\begin{align*}
		\tilde{\textbf{d}}^{k+1}&=\textbf{d}^k+\sigma K\bar{\textbf{u}}\,,\\
		\textbf{d}^{k+1}&=\argmin_{\textbf{d}}\Bigg\{\frac{1}{2}\|\textbf{d}-\tilde{\textbf{d}}^{k+1}\|_2^2+\sigma F^*(\textbf{d})\Bigg\}\,,\\
		\tilde{\textbf{u}}^{k+1}&=\textbf{u}^k-\tau K^*\textbf{d}^{k+1}\,,\\
		\textbf{u}^{k+1}&=\argmin_{\textbf{u}}\Bigg\{\frac{1}{2}\|\textbf{u}-\tilde{\textbf{u}}^{k+1}\|_2^2+\tau G(\textbf{u})\Bigg\}\,,\\
		\bar{\textbf{u}}_{k+1} &=\textbf{u}_{k+1}+\theta(\textbf{u}_{k+1}-\textbf{u}_k) \qquad \text{(over-relaxation)} \; ,
	\end{align*}
where $\tau,\sigma$ are the parameters associated with the primal and dual variables, $\theta$ is the over-relaxation parameter.
	\subsection{The Primal-Dual Formulation}
	In this case we have
	\[
	G(\textbf{u})= \int_\Omega |f_t+\nabla f\cdot \textbf{u}|, \quad F(K\textbf{u})=\frac{\eta}{2}\int_\Omega \phi(|\nabla f|)(\nabla\cdot\textbf{u})^2 +\gamma \sum_{i=1}^2 \int_\Omega |\nabla u_i|\,.
	\]
	The Operator $K$ is given as
	\begin{equation*}
		K\textbf{u}=\begin{bmatrix}
			\nabla & 0 \\
			0 & \nabla \\
			\phi\partial_x & \phi\partial_y
		\end{bmatrix}
		\begin{bmatrix}
			u_1\\[0.4cm] u_2
		\end{bmatrix}.
	\end{equation*}
	Therefore,
	\begin{equation*}
		K^*\textbf{d} = -\begin{bmatrix}
			\nabla\cdot & 0 & \partial_x(\phi\cdot) \\
			0 & \nabla\cdot & \partial_y(\phi\cdot)
		\end{bmatrix}
		\begin{bmatrix}
			d_1\\
			d_2\\
			d_3
		\end{bmatrix}.
	\end{equation*}
	The convex conjugate $F^*(\textbf{d})$ is computed as
	\begin{equation*}
		F^*(\textbf{d}) = \frac{1}{2\eta}\|d_3\|^2_2 +\gamma\sum_{i=1}^2 \delta_{B(L^\infty)}(d_i/\gamma)\,,
	\end{equation*}
	Thus the primal-dual formulation is given as
	\begin{equation*}
		\argmin_\textbf{u} \argmax_{\textbf{d}}\:\: |f_t+\nabla f\cdot \textbf{u}|_1 + ( \textbf{u},K^*\textbf{d}) -\frac{1}{2\eta}\|d_3\|^2_2 -\gamma\sum_{i=1}^2 \delta_{B(L^\infty)}(d_i/\gamma)\,.
	\end{equation*}
	Accordingly, the Chambolle-Pock algorithm for this primal-dual problem is given as:
	\begin{align*}
		\tilde{\textbf{d}}^{k+1}&=\textbf{d}^k+\sigma K\bar{\textbf{u}}\,,\\
		\textbf{d}_{1,2}^{k+1}&=\argmin_{\textbf{d}}\Bigg\{\frac{1}{2}\|\textbf{d}-\tilde{\textbf{d}}_{1,2}^{k+1}\|_2^2+\gamma\sigma\delta_{B(L^\infty)}(\textbf{d}/\gamma)\Bigg\}\,,\\
		d_{3}^{k+1}&=\argmin_{d}\Bigg\{\frac{1}{2}\|d-\tilde{d}_{3}^{k+1}\|_2^2+\frac{\sigma}{2\eta}\|d\|_2^2\Bigg\}\,,\\
		\tilde{\textbf{u}}^{k+1}&=\textbf{u}^k-\tau K^*\textbf{d}^{k+1}\,,\\
		\textbf{u}^{k+1}&=\argmin_{\textbf{u}}\Bigg\{\frac{1}{2}\|\textbf{u}-\tilde{\textbf{u}}^{k+1}\|_2^2+\frac{\tau}{2}|f_t+\nabla f\cdot \textbf{u}|\Bigg\}\,,\\
		\bar{\textbf{u}}_{k+1} &=\textbf{u}_{k+1}+\theta(\textbf{u}_{k+1}-\textbf{u}_k).
	\end{align*}
	
	\subsection{The Primal-Dual Algorithm} 
	The $L^1$ norm of the OFC is an affine transformation of the flow $\textbf{u}$ which can be solved directly by soft thresholding formula (see \cite{dirks}) as:
	\begin{equation}
		\textbf{u}^{k+1} = \tilde{\textbf{u}}^{k+1}+
		\begin{cases}
			\tau\nabla f, \qquad\qquad\qquad\qquad\quad \text{ if } f_t+\nabla f\cdot \tilde{\textbf{u}}^{k+1} < -\tau|\nabla f|^2 \\[0.3cm]
			-\tau\nabla f, \qquad\qquad\qquad\quad\quad \text{ if } f_t+\nabla f\cdot \tilde{\textbf{u}}^{k+1} > \tau|\nabla f|^2 \\[0.3cm]
			(f_t+\nabla f\cdot \tilde{\textbf{u}}^{k+1})\dfrac{\nabla f}{|\nabla f|^2}, \quad \text{ otherwise}
		\end{cases}.
		\label{eq}
	\end{equation}
	The optimality conditions for the dual variables can be derived using similar computations in \cite{hirak}. The dual update for $d_{1,2}$ can be obtained by the point-wise projection of $\tilde{d}_{1,2}$ onto $[-\gamma,\gamma]$ given as
	\begin{equation*}
		d_{1,2} = \proj\nolimits_{\sigma/\gamma}(\tilde{d}_{1,2}) = \min(\gamma,\max(-\gamma, \tilde{d}_{1,2})).
	\end{equation*}
The sub-problem for $d_3$ is a linear quadratic minimization problem. The optimality conditions are derived as follows. Let us consider the functional
\begin{equation*}
	J(d_3)=\int_\Omega \frac{1}{2}(d_3-\tilde{d}_3)^2+\frac{\sigma}{2\eta}\int_\Omega (d_3)^2.
\end{equation*}
By computing the directional derivative of $J$ we obtain
\begin{equation*}
	\frac{d}{d\alpha} J(d_3+\alpha \xi)=\int_\Omega \xi(d_3+\alpha\xi-\tilde{d}_3+\frac{\sigma}{\eta}(d_3+\alpha\xi)).
\end{equation*}
The critical points are obtained by evaluating $\frac{d}{d\alpha} J(d_3+\alpha \xi)\Big|_{\alpha=0}=0$. This gives
\begin{equation*}
	\int_\Omega \xi(d_3-\tilde{d}_3+\frac{\sigma}{\eta}d_3)=0.
\end{equation*}
Since $\xi$ is arbitrary we obtain $d_3-\tilde{d}_3+\frac{\sigma}{\eta}d_3=0$. Thus the Chambolle-Pock scheme can be written as:
\begin{align*}
	\tilde{\textbf{d}}^{k+1}&=\textbf{d}^k+\sigma K\bar{\textbf{u}}\,,\\
	\textbf{d}_{1,2}^{k+1}&=\proj\nolimits_\gamma\Big(\tilde{\textbf{d}}_{1,2}^{k+1}\Big)\,,\\
	d_3^{k+1}&=\frac{\eta}{\eta+\sigma}\tilde{d}_{3}^{k+1}\,,\\
	\tilde{\textbf{u}}^{k+1}&=\textbf{u}^k-\tau K^*\textbf{d}^{k+1}\,,\\
	\textbf{u}^{k+1} &= \tilde{\textbf{u}}^{k+1}+
	\begin{cases}
		\tau\nabla f, \qquad\qquad\qquad\qquad\quad \text{ if } f_t+\nabla f\cdot \tilde{\textbf{u}}^{k+1} < -\tau|\nabla f|^2 \\[0.3cm]
		-\tau\nabla f, \qquad\qquad\qquad\quad\quad \text{ if } f_t+\nabla f\cdot \tilde{\textbf{u}}^{k+1} > \tau|\nabla f|^2 \\[0.3cm]
		(f_t+\nabla f\cdot \tilde{\textbf{u}}^{k+1})\dfrac{\nabla f}{|\nabla f|^2}, \quad \text{ otherwise}
	\end{cases},\\
	\bar{\textbf{u}}_{k+1} &=2\textbf{u}_{k+1}-\textbf{u}_k.
\end{align*}
The algorithm for obtaining the flow field using this scheme is presented below.
	\begin{algorithm}[H]
		\caption{}\label{pd1}
		\begin{algorithmic}[1]
			\State Define $\tau,\sigma$
			\State Initialize $\textbf{u}^0\gets$ 0, $\textbf{d}^0\gets$ 0
			\State Initialize matrix $K$
			\Repeat
			\State $\textbf{u}_\text{old}\gets\textbf{u}$
			\State $\tilde{\textbf{d}}\gets \textbf{d}+\sigma K\bar{\textbf{u}}$
			\State $d_{1,2}\gets\proj_{\sigma/\gamma}(\tilde{d}_{1,2})$
			\State $d_3\gets \frac{\eta}{\eta+\sigma}\tilde{d}_3$
			\State Compute matrix $K^*$
			\State $\tilde{\textbf{u}}\gets \textbf{u}-\tau K^*\textbf{d}$
			\State $\textbf{u}\gets \text{softThreshold}(\tilde{\textbf{u}})$
			\State $\bar{\textbf{u}}\gets 2\textbf{u}-\textbf{u}_\text{old}$
			\Until{convergence}
		\end{algorithmic}
	\end{algorithm}
	Chambolle and Pock \cite{cham} also showed that the convergence criterion is fulfilled when $\tau\sigma\|K\|^2< 1,\theta = 1$. We will discuss the choice of $\tau$ and $\sigma$ used in our algorithm in the subsequent section.
	\subsection{Primal-Dual Residuals}
	An useful error metric for the primal-dual algorithm can be obtained from the primal and dual residuals. Let $\textbf{u}^k,\textbf{d}^k$ denote the primal and dual iterates after $k$ iterations. Define $\textbf{u}_{e}:=\textbf{u}^k-\textbf{u}^{k+1}, \textbf{d}_{e}:=\textbf{d}^k-\textbf{d}^{k+1}$ as the error between successive iterates for the primal and dual variables respectively. The primal and dual residuals are defined as:
	\begin{equation*}
		p_{r}^{(k)} := \Big|\frac{\textbf{u}_{e}}{\tau} - K^*\textbf{d}_{e}\Big|,\qquad
		d_{r}^{(k)} := \Big|\frac{\textbf{d}_{e}}{\sigma} - K\textbf{u}_{e}\Big|.
	\end{equation*}
	The normalized error at the $k^{th}$ iteration is then computed as:
	\begin{equation*}
		e^{(k)}=\dfrac{p_{r}^{(k)}+d_{r}^{(k)}}{|\Omega|},
	\end{equation*}
	where $|\Omega|$ is the dimension of the domain $\Omega$.
\section{Numerical Optimization}
\label{sec:5}
In this section we will discuss some of the best practices for improving optical flow estimation. Further, we will discuss different filtering choices used in our work that contribute towards further improvement of the accuracy of the flow.
\subsection{Established Implementation Practices}
There are several factors that contribute significantly to the accuracy of optical flow estimation. These are simple optimization practices that make the optical flow estimation highly accurate and efficient \cite{brox,bruhn,sun1}.

A coarse-to-fine pyramidal grid is employed in the implementation scheme to account for pixels with larger displacements. The first step is the computation of flow at the coarsest level. The estimates obtained at this level are projected to the next finer grid. Using these estimates the second image is warped towards the first image by bi-cubic interpolation. The flow increment is then computed between the first image and the warped image. This process is repeated till the finest level in the grid is reached.

Another important optimization practice is the efficient computation of spatial and temporal image derivatives. The derivative of the second image is computed using a 5-point derivative filter. Using the current flow estimates the second image and its derivative is warped towards the first image using bi-cubic interpolation \cite{sun2}. This step is a departure from the conventional techniques as it involves the flow estimates in the computation of image derivatives. The time derivative is simply the difference between the first image and the warped image. The spatial derivatives are the weighted average of the warped image derivatives and the spatial derivatives of the first image,
\begin{align*}
	(f_x)_{avg} &= r*f^w_x+(1-r)*f_x \\[0.3cm]
	(f_y)_{avg} &= r*f^w_y+(1-r)*f_y,
\end{align*}
where $f^w_x,f^w_y$ denote the spatial derivatives of the warped image, $f_x,f_y$ denote the spatial derivatives of the first image, $r\in (0,1)$ is called the blending ratio.

\subsection{Different Filtering Choices for Improving Estimation}
While the above-mentioned steps significantly contribute to the improvement of the optical flow estimation, they can also generate small outliers in the flow due to interpolation error. To compensate for this issue median filtering is performed at every warping iteration. The median filter effectively removes all the spurious outliers leading to significant improvement in accuracy. However, as discussed before, there are two main drawbacks of the standard median filtering. First it leads to higher energy solutions \cite{sun2} and secondly they have the same asymptotic worst-case mean-squared-error as linear filtering when the SNR is of the order 1 \cite{donoho}. To overcome these challenges we use the heuristic of iterated median filtering over standard median filter to remove outliers. Additionally we will use the weighted filtering principle as a post-processing step to refine the flow edges.
\subsubsection{Iterated Median Filtering}
The concept of iterated median filtering relies upon repeated applications of median filters. Formally we can say that
\begin{equation*}
	M_{\text{iter}}(f): = \med_{h_k}\circ\dots\circ\med_{h_1}(f).
\end{equation*}
Here $f$ is the given image, $h_i, i=1,\dots,k$ is the window size, $\circ$ denotes the composition operation. For each pass the window size can be chosen accordingly. According to \cite{sun1}, an optimal filter size is $5\times 5$. Let us describe the process for two-stage median filtering. Let $f_1$ denote the coarser version of the image $f$. Let $h_1$ and $h_2$ be the window sizes of the filter. The process can be described in three steps:
\begin{enumerate}
	\item $M_{h_1}(f_1):=\med_{h_1}(f_1)$.
	\item upscale $M_{h_1}(f_1)$ using interpolation.
	\item $M_{h_1,h_2}(f)=\med_{h_2}(M_{h_1}(f_1)).$
\end{enumerate}
The median filter is first applied at the coarser level. The resulting image is upsampled to the next finer level where another pass of median filtering is applied.
\begin{figure}[H]%
	\centering
	\subfloat[Noisy image (Grove 2)]{{\includegraphics[width=4.0cm]{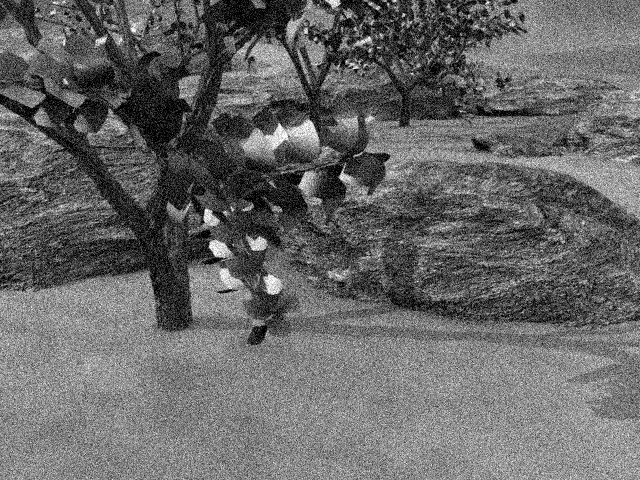}}}%
	\qquad
	\subfloat[Denoised image with Median Filtering, PSNR = 25.12]{{\includegraphics[width=4.0cm]{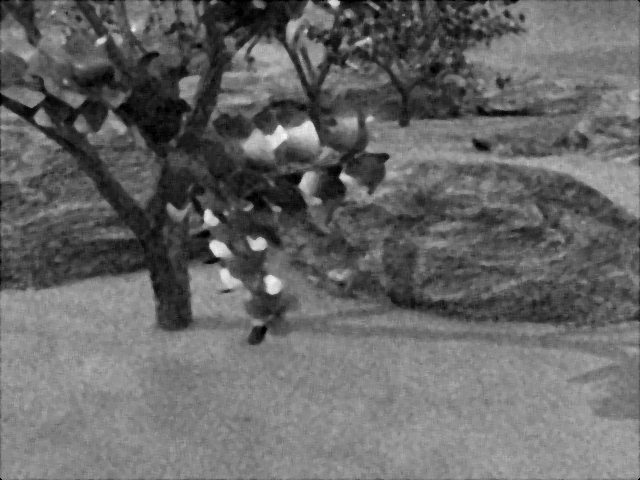}}}%
	\qquad
	\subfloat[Denoised image with Iterated Median Filtering, PSNR = 26.06]{{\includegraphics[width=4.0cm]{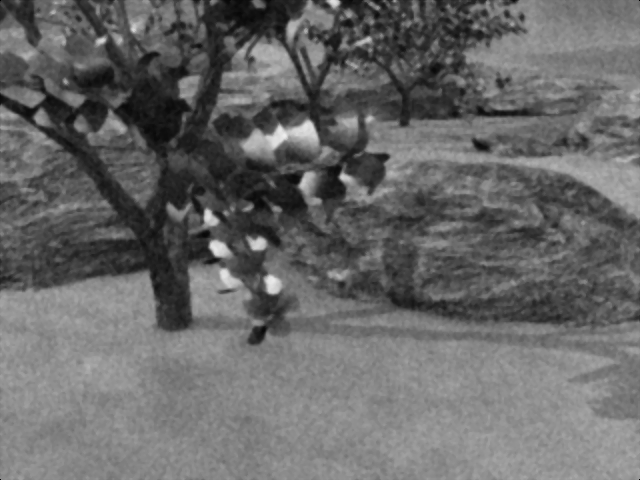}}}\\
	\subfloat[Noisy image (Rubberwhale)]{{\includegraphics[width=4.0cm]{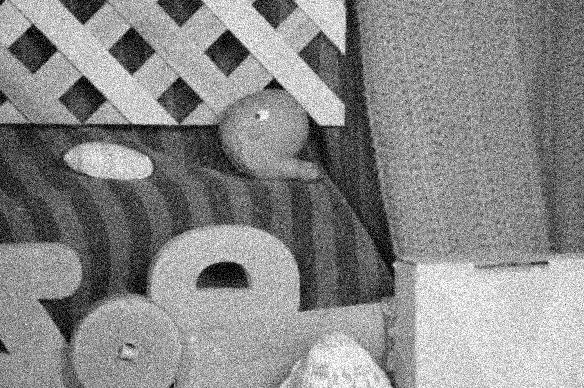}}}%
	\qquad
	\subfloat[Denoised image with Median Filtering, PSNR = 28.87]{{\includegraphics[width=4.0cm]{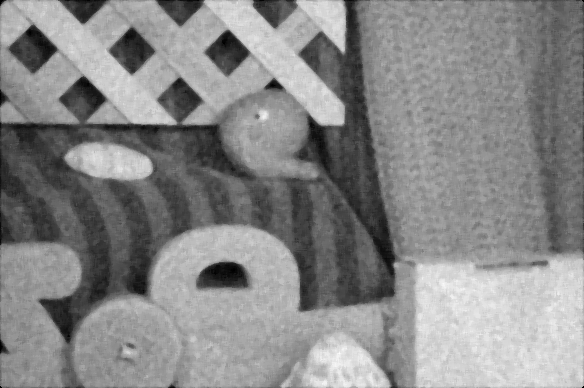}}}%
	\qquad
	\subfloat[Denoised image with Iterated Median Filtering, PSNR = 30.05]{{\includegraphics[width=4.0cm]{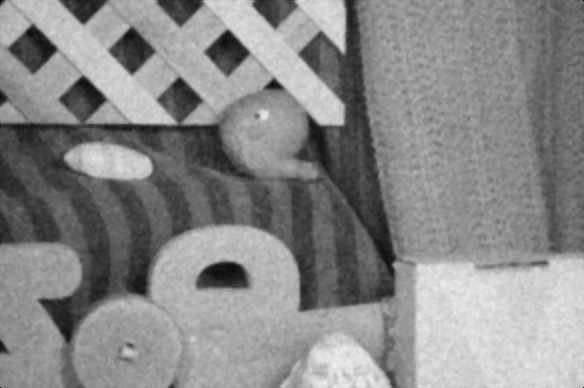}}}
	\caption{Behaviour of two filtering technqiues under \emph{Gaussian} noise with variance 0.1.}
	\label{f9}
\end{figure}
We performed experiments on a few images namely the \emph{grove 2} and \emph{rubberwhale} from the middlebury datasets under varying noise densities. We used the PSNR metric for a qualitative evaluation. The results in Figure (\ref{f9}) show that the iterated median filtering produces better results.

To further demonstrate the advantages of iterated median filtering over standard median filtering for noisy image sequences in optical flow problems, we created a synthetic image sequence of size $200\times200$. A circular ball of radius 100 pixels centered at the origin is superimposed on the background. The second frame is obtained from the first frame by following a translational motion of this circular ball to the right by 4 pixels.
\begin{figure}[H]%
	\centering
	\subfloat{{\includegraphics[width=4cm]{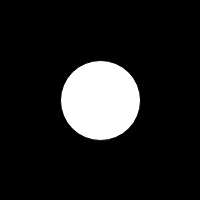}}}%
	\qquad
	\subfloat{{\includegraphics[width=4cm]{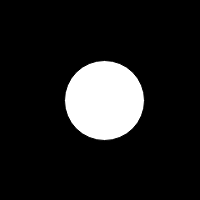}}}%
	\caption{Synthetic sequence}%
	\label{f10}%
\end{figure}
For convenience, the following demonstration is shown for the standard $L^2-TV$ optical flow algorithm. It relies upon an useful approximation of $\Delta_1$ discussed in \cite{adam}, where the author present an explicit convergent, monotone scheme given by
\begin{equation*}
	\Delta_1u\approx 2\frac{u(x)-u_*}{dx^2}+\mathcal{O}(dx^2+d\theta)
\end{equation*}
where $u(x)$ is the value at the centre of the grid, $u_*$ denotes the median value of the neighbors in the grid, $dx$ is the spatial resolution and $d\theta$ is the directional resolution. For images, which in general have a natural discrete structure, $d\theta$ is fixed at $\pi/4$.

\begin{figure}[H]%
	\hspace*{-1.5cm}
	\subfloat[Standard median filtering with gaussian noise]{{\includegraphics[width=7cm]{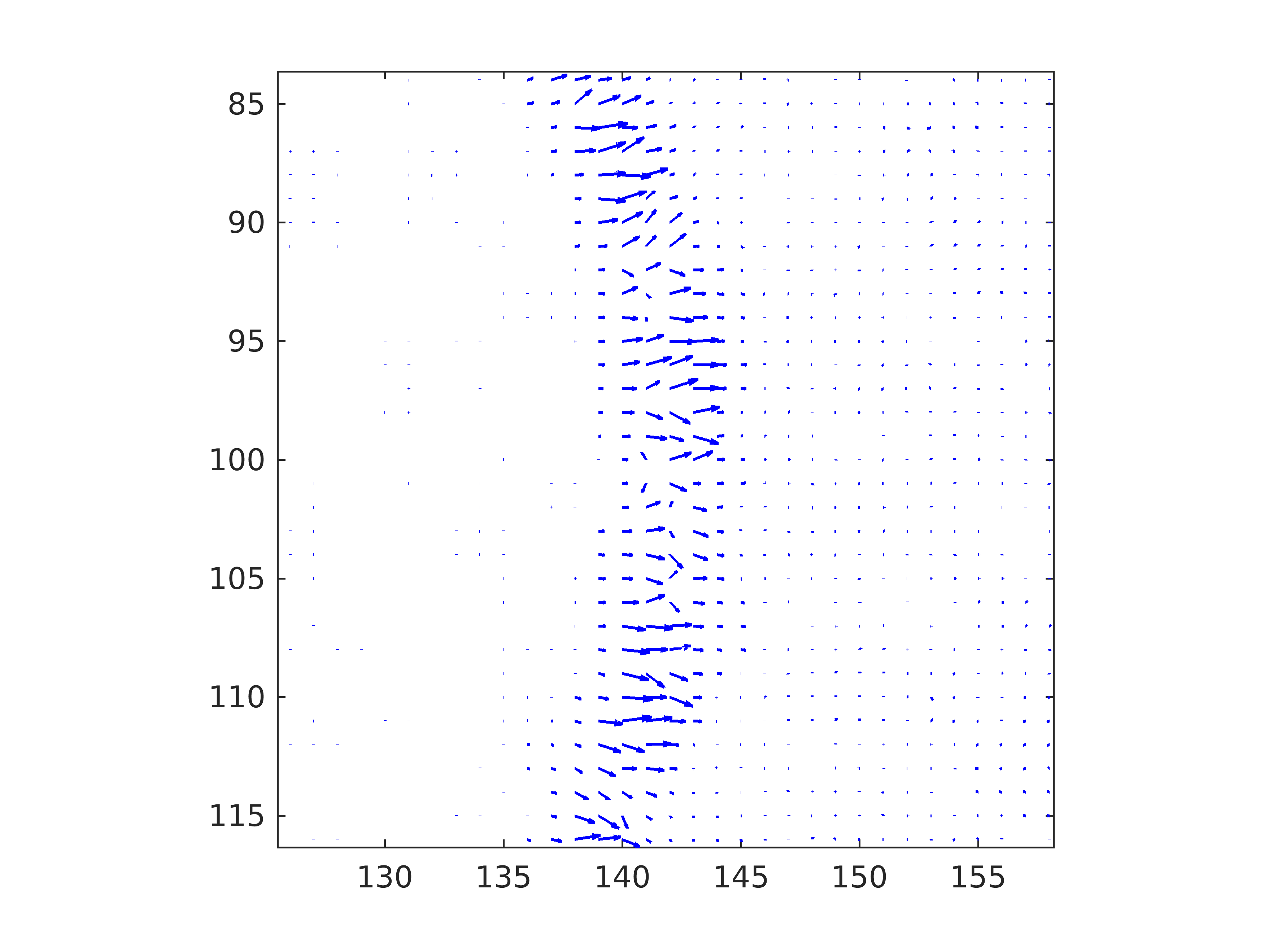}}}%
	\qquad
	\subfloat[Iterated median filtering with Gaussian noise]{{\includegraphics[width=7cm]{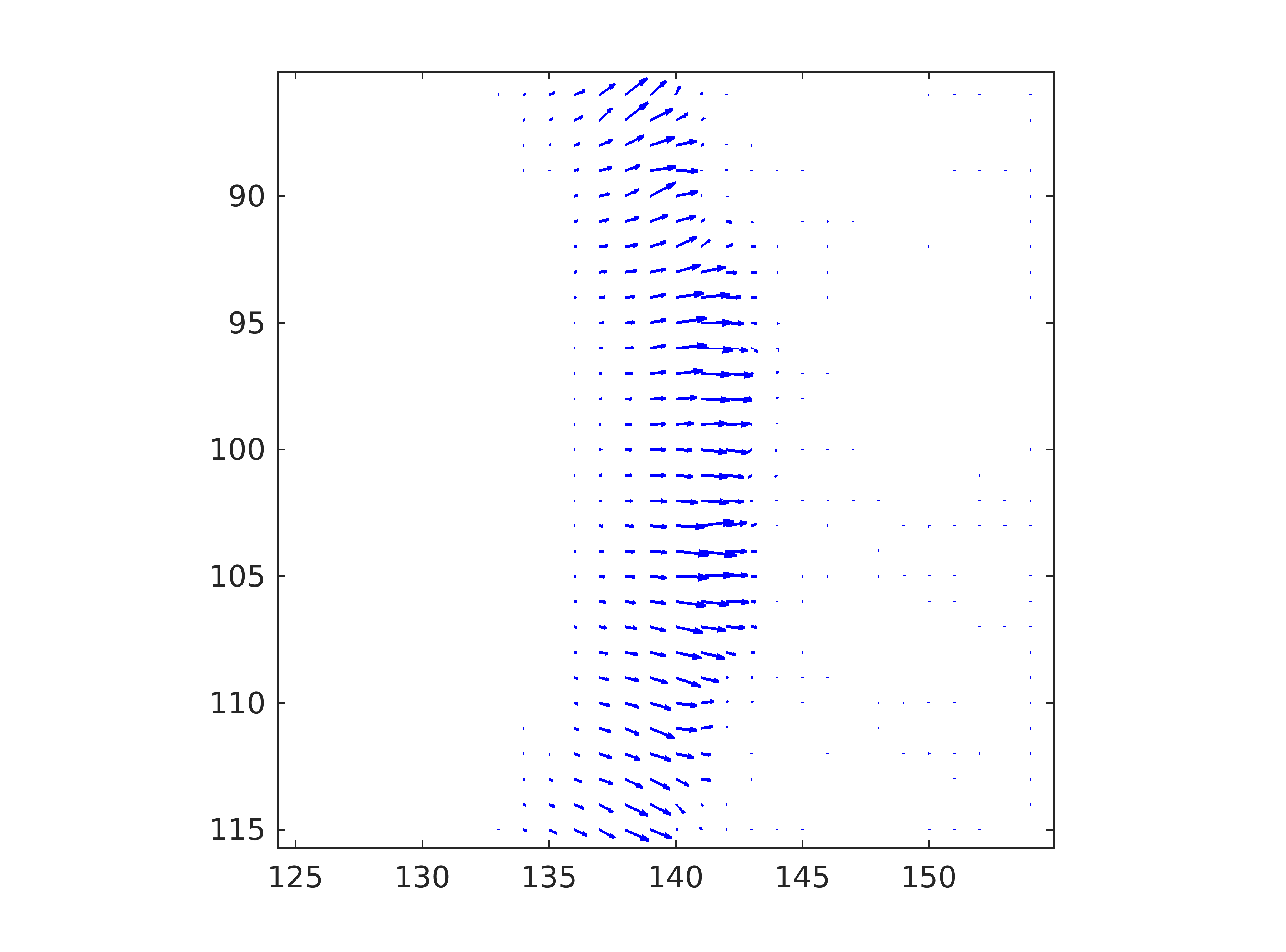}}}%
	\qquad
	\hspace*{-0.2cm}
	\subfloat[Standard median filtering with salt and pepper noise]{{\includegraphics[width=7cm]{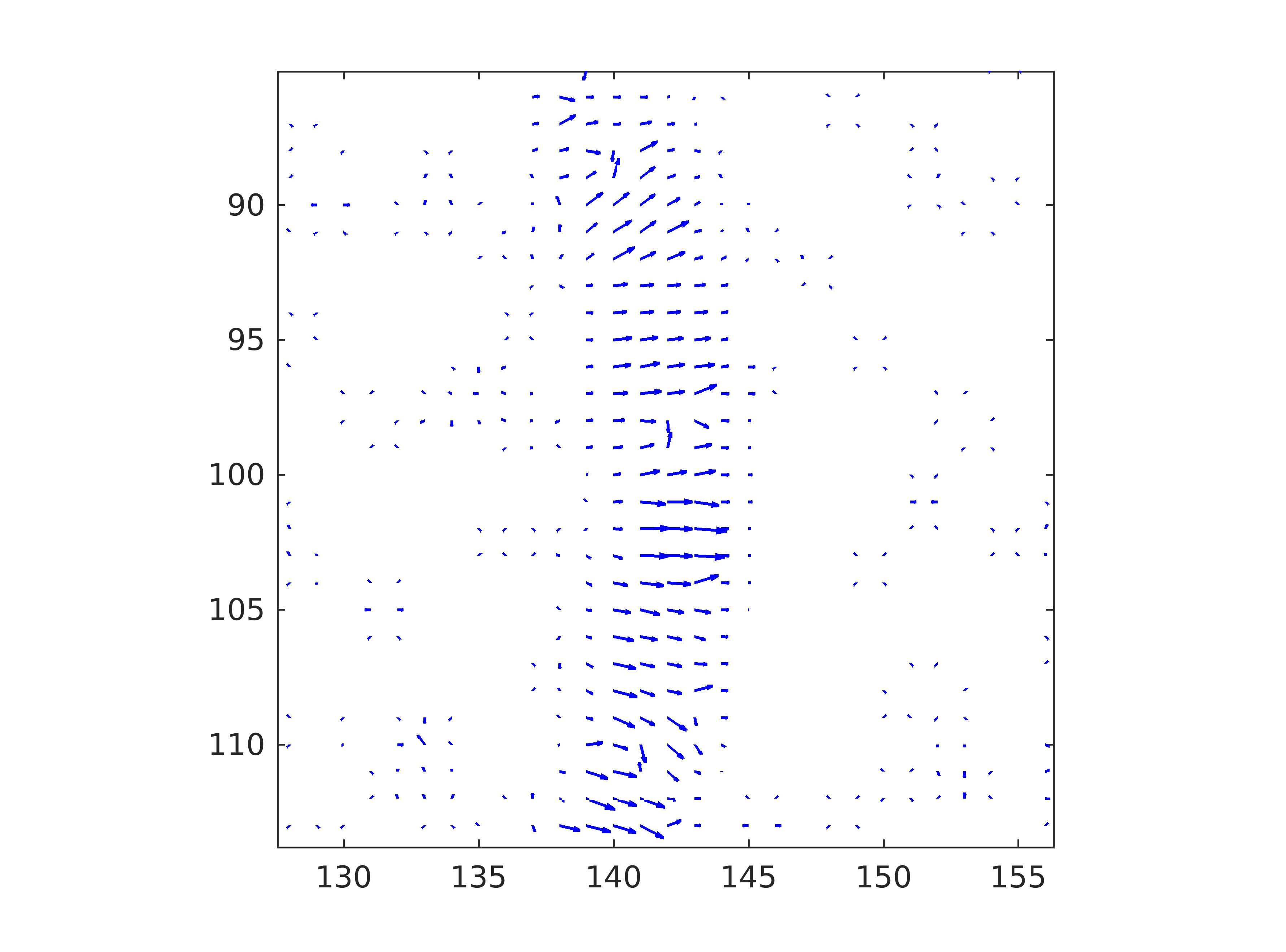}}}%
	\qquad
	\subfloat[Iterated median filtering with salt and pepper noise]{{\includegraphics[width=7cm]{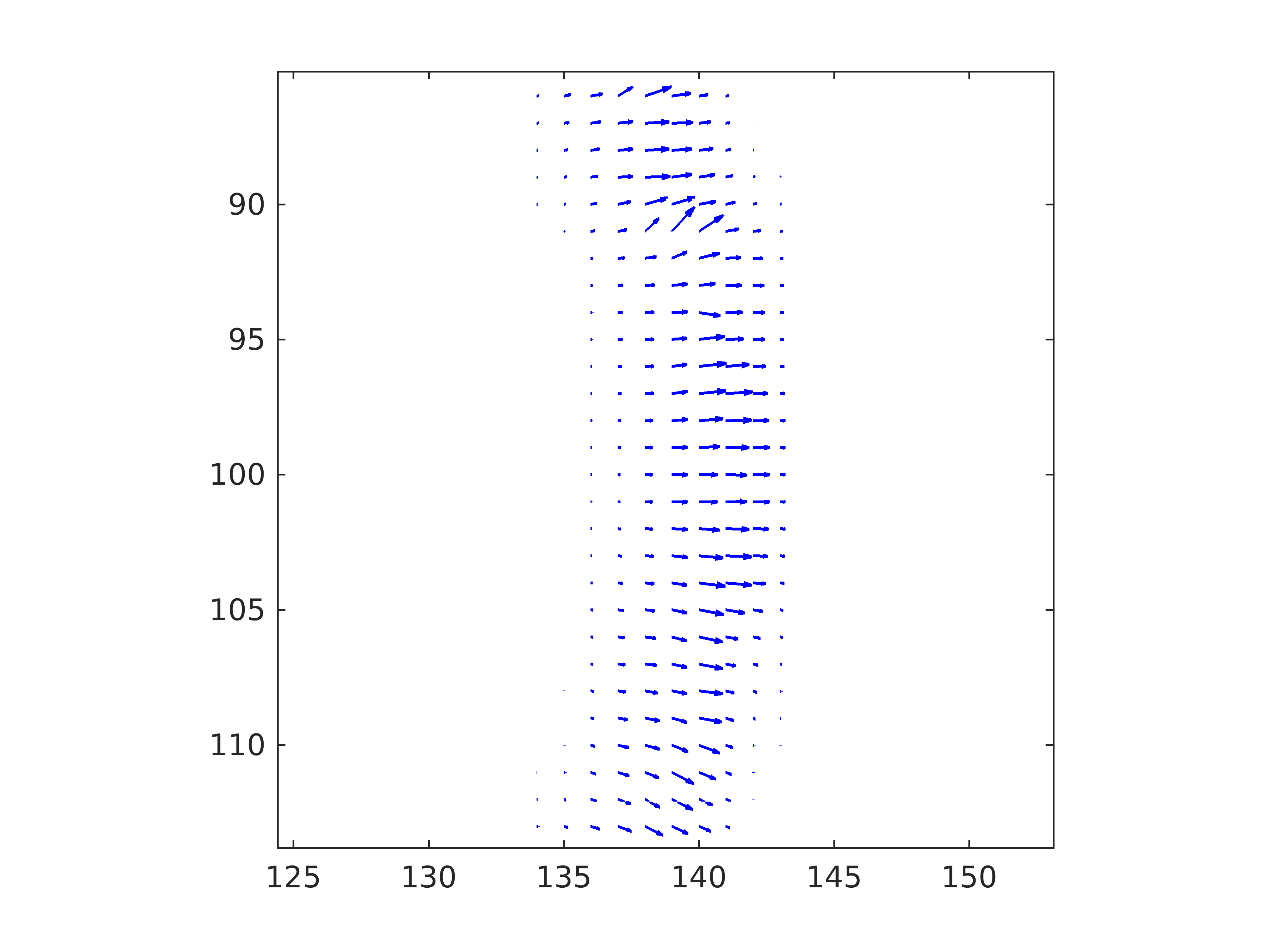}}}
	\caption{Comparison of standard median filtering with iterated median filtering with different noise levels.}
	\label{f1}
\end{figure}
From the results above it can be clearly seen that the iterated median filtering technique clearly outperforms standard median filtering in the presence of noise. Additionally it preserves the direction of motion vectors.
\subsubsection{Weighted Median Filtering}
To deal with the higher energy solutions by application of median filter, an optimization technique is employed where the classical objective function is modified with a new \emph{non-local} term, \cite{sun1,sun2,sana,zhang2}. The \emph{non-local} term governs the smoothness within a specified region. These models are classified under NL-based models. Li and Osher \cite{li} established a formal connection between the \emph{non-local} term and weighted median filtering for $L^1$-based energy minimization.

The weighted median filtering technique aims to minimize an objective of the form
\begin{equation*}
	\min_x E(x):=\sum_{i} w_i (x-u_i)+F(x),
\end{equation*}
where $w_i$ are the non-negative weights, $u_i$ are non-negative values. For a specific application to PDE-based image denoising, Li and Osher suggested the following weight function:
\begin{equation*}
	w(x,y)=\exp\Bigg(-\frac{1}{h^2}\int_\Omega G_\delta(t)|f(x+t)-f(y+t)| dt\Bigg),
\end{equation*} 
where $G_\delta$ is the Gaussian with the standard deviation $\delta$, $f$ is the input image, $h$ is the grid size. In our work, we do not modify the objective functional directly. Instead, we use the weighted median filter as a post-processing step to refine the flow estimate obtained by solving the primal-dual algorithm during coarse-to-fine estimation.
\begin{figure}[H]%
	\centering
	\subfloat{{\includegraphics[width=7cm]{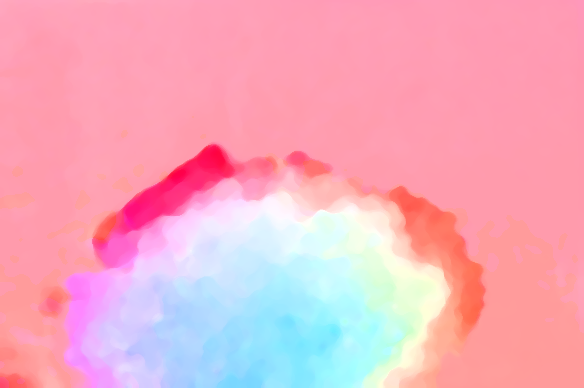}}}%
	\qquad
	\subfloat{{\includegraphics[width=7cm]{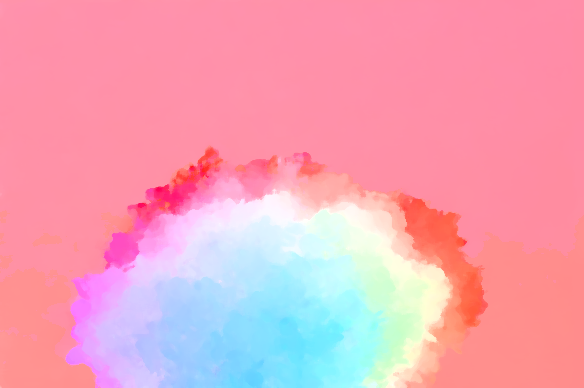}}}%
	\caption{Improving flow edges using weighted median filtering.}%
	\label{f8}%
\end{figure}
Figure (\ref{f8}) shows the improvement of the flow edges for the \emph{hydrangea} sequence using weighted median filter.
\section{Experiments and Results}
\label{sec:6}
\subsection{Evaluation Metrics}
The most commonly used metrics in the literature for a quantitative evaluation of optical flow methods are the Average Angular Error (AAE) and the average End Point Error (EPE). The AAE is computed as
\begin{equation*}
	AAE = \frac{1}{|\Omega|}\int_\Omega \frac{\textbf{u}_c\cdot\textbf{u}_e}{|\textbf{u}_c||\textbf{u}_e|}\:dxdy,
\end{equation*}
where $\textbf{u}_c$ is the computed flow and $\textbf{u}_e$ is the exact flow. This metric was first used by Barron et.al. \cite{barron} in their pioneering work where they evaluted the performance of several existing optical flow models. The EPE is computed as
\begin{equation*}
	EPE = |\textbf{u}_e-\textbf{u}_c|=\sqrt{(u_1^e-u_1^c)^2 + (u_2^e-u_2^c)^2},
\end{equation*}
where $(u_1^e,u_2^e)$ is the exact optical flow and $(u_1^c,u_2^c)$ is the computed optical flow.

\subsection{Experiments on Middlebury Dataset}
The middlebury dataset is one of the most commonly used dataset for evaluating various optical flow methods. There are 8 sequences in this dataset with a publicly available ground-truth information.
\begin{figure}[H]%
	\centering
	\subfloat{{\includegraphics[width=5.5cm]{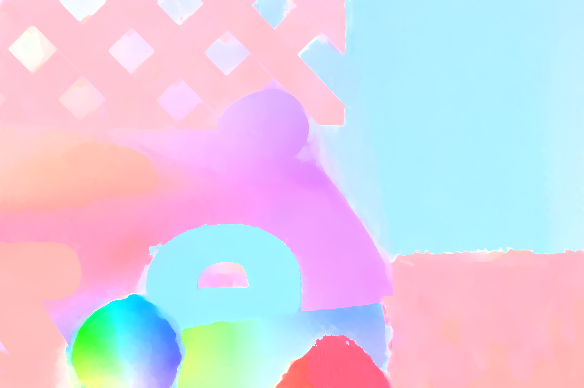}}}%
	\quad
	\subfloat{{\includegraphics[width=5.5cm]{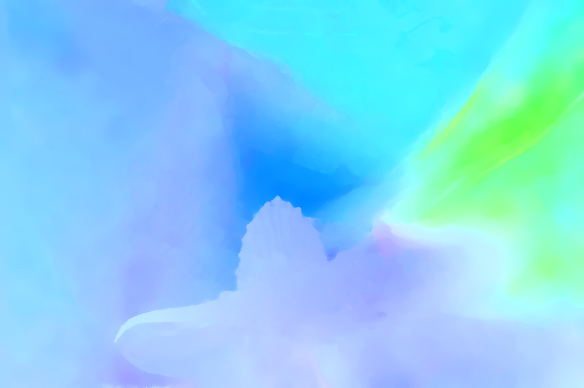}}}
	\quad
	\subfloat{{\includegraphics[width=5.5cm]{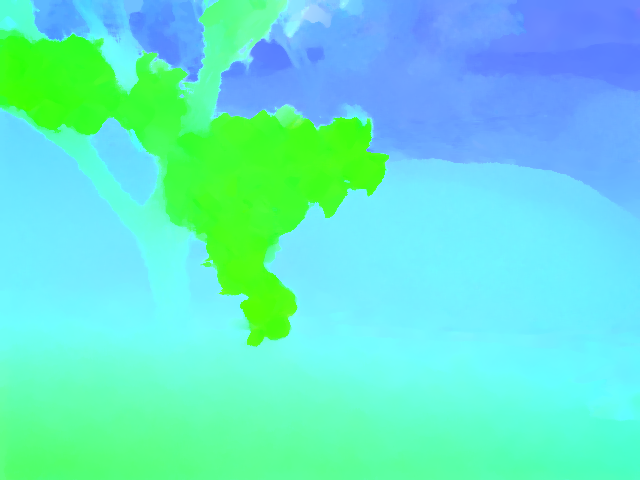}}}
	\caption{Estimated flow field for the \emph{rubberwhale}, \emph{dimetrodon} and \emph{grove2} sequence.}%
	\label{f4}%
\end{figure}
\begin{table}[H]
	\renewcommand{\arraystretch}{2.0}
	\begin{adjustbox}{width=1\textwidth}
		\begin{tabular}{p{3.0cm}cccccccc}
			\hline
			\multirow{2}{7cm}{} & Dimetrodon &Rubberwhale & Hydrangea & Urban 2 & Urban3 & Grove 2 & Grove 3 & Venus \\
			  & AAE / EPE  & AAE / EPE & AAE / EPE & AAE / EPE & AAE / EPE & AAE / EPE & AAE / EPE & AAE/EPE\\
			\hline
			HS+NL\cite{sun2}  & 4.733 / 0.230 & 4.992 / 0.154 & 2.890 / 0.250 & 4.847/0.566 & 6.659/0.742 & 2.580/0.185 & 6.187/0.640 & 5.498/0.333\\
			HS+NL+GF\cite{zhang2}  & 4.278 / 0.208 & 4.667 /0.143 &2.567 / 0.430 & 4.525/0.547 & 6.521/\textbf{0.723} & \textbf{2.360}/\textbf{0.166} & \textbf{6.005}/\textbf{0.632} & 5.140/0.310\\ 
			Our method  &  \textbf{2.805} / \textbf{0.142} & \textbf{2.989} / \textbf{0.100} & \textbf{2.135} /  \textbf{0.191} & \textbf{2.997}/\textbf{0.409} & \textbf{5.788}/0.858 & 2.885/0.195 & 6.871/0.716 &\textbf{3.861}/\textbf{0.280}  \\ \hline
		\end{tabular}
	\end{adjustbox}
	\caption{\small{Comparison of the Average Angular Error (AAE) and End Point Error (EPE) on the midddlebury sequences.}}
	\label{t1}
\end{table}
For a more convinving comparison, we compared our results with the some of the state-of-the-art modified HS-based formulations. The first model is the HS model modified by introducing the \emph{non-local} term \cite{sun2} combined with the state-of-the-art implementation practices. The second model is an improvement of the first model by the application of guided filtering (GF) techniques, see \cite{zhang2}. The results obtained are indicated in Table \ref{t1}. The lowest error obtained in each case are indicated in bold. Table \ref{t2} shows that our method achieves the lowest average error compared to the other methods.

\begin{table}[H]
	\centering
\begin{tabular}{cc}
	\hline
	\multirow{2}{3cm}{} & Average \\
	& AAE / EPE \\
	\hline
	HS+NL\cite{sun2} & 4.798/0.388 \\[0.3cm]
	HS+NL+GF\cite{zhang2} & 4.508/0.370\\[0.3cm]
	Our Method & \textbf{3.791}/\textbf{0.362}\\
	\hline
\end{tabular}
\caption{Average errors of the compared methods on the middlebury dataset.}
\label{t2}
\end{table}
\subsection{Choice of Parameters}
To obtain the desired edge-preserving flow estimate requires fine-tuning of multiple free parameters. In this section we will discuss about some of the choices of free parameters in our algorithm.

For the primal-dual algorithm, we performed experiments with different values of the parameters, $\tau$ and $\sigma$ ranging from $0.1$ to $1$. Chambolle and Pock \cite{cham} showed that the numerical scheme converges when $\tau\sigma\|K\|^2<1$. The best results were obtained for $\tau=1,\sigma =0.9$.

\begin{figure}[H]%
	\hspace*{-1.4cm}
	\subfloat{{\includegraphics[width=6cm]{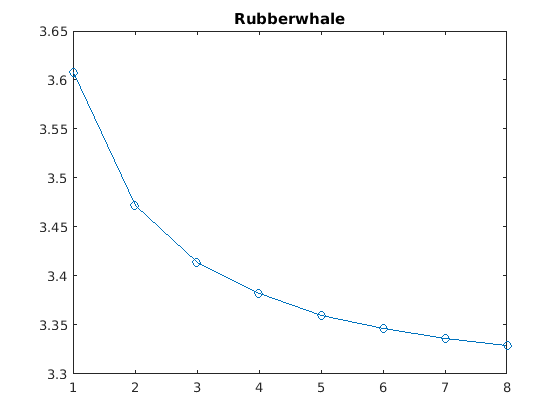}}}%
	\subfloat{{\includegraphics[width=6cm]{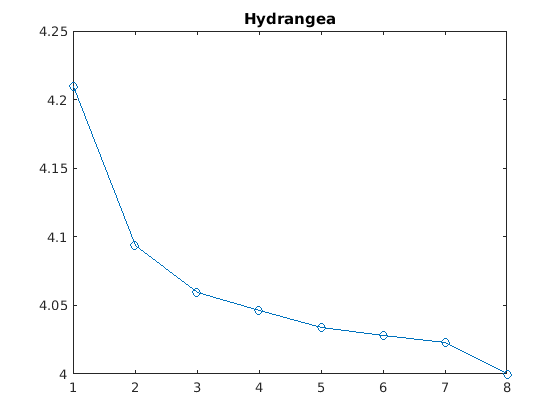}}}%
	\subfloat{{\includegraphics[width=6cm]{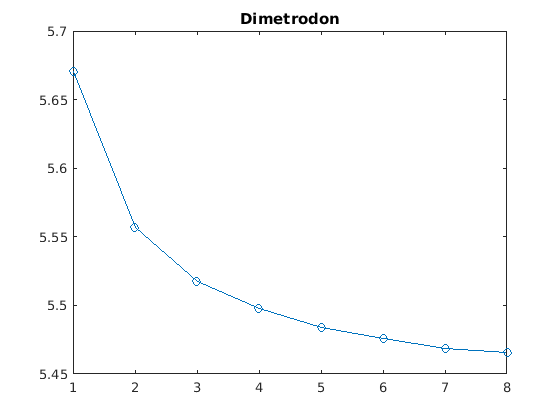}}}
	\caption{Normalized error computed with the primal-dual residuals for the \emph{rubberwhale}, \emph{hydrangea} and \emph{dimetrodon} sequence.}%
	\label{f6}%
\end{figure}
Chambolle and Pock further introduced a primal-dual gap $\mathcal{G}$ for convergence analysis. They observed that an order $\mathcal{O}(1/N)$ can be achieved for $\mathcal{G}$ when $G$ and $F^*$ have full domain.

For implementing the iterated median filtering, we experimented with different window sizes $3,5,7$ during the coarse-to-fine process. We observed that a window size of $5$ at the coarse level and a window size of $3$ at the finer level produced best results.

Experiments were done with different values of $\gamma$ and $\eta$ ranging from 0.01 to 10. A higher value of $\eta$ led to over-smoothening of the edges. Keeping $\gamma$ fixed we varied $\eta$ between $0.01, 0.1, 0.5, 1, 10$. The best results were obtained for $\eta = 0.01$. Next for this $\eta$ we varied $\gamma$ and found the optimal choice for $\gamma=1$.

The pyramid levels can be adaptively determined by the formula, see \cite{sun1},
\begin{equation*}
	P_{\text{lev}} = 1+\left\lfloor\frac{\log\Big(\frac{\min\{m,n\}}{16}\Big)}{\log(P_{\text{sp}})}\right\rfloor,
\end{equation*}
where $P_{\text{sp}}$ denotes the pyramid spacing, $m,n$ denote the flow dimensions at each pyramid level and $\left\lfloor\cdot\right\rfloor$ indicates the floor function. We used $P_{\text{lev}}=5$ in our experiments. At each pyramid level, 10 warping iterations were performed.

For the weighted median filter, we used the weight $w$ specified by Li and Osher. The weights are restricted to a search window $|x-y|_\infty\le R$, where denotes the window size, see \cite{li}. We performed several experiments on $(\delta,R)$, where $\delta$ is the standard deviation defined previously. A candidate set for R was $4,5,7,9,10,11,13$. The best results for the following values: $(13,7)$ for \emph{grove2} sequence, $(4,70)$ for \emph{grove3} sequence and $(10,7)$ for the remaining sequences.

\section*{Conclusion}
In this work, we have proposed a variational optical flow model for an efficient edge-preserving optical flow estimation. An effective numerical scheme was developed using the Chambolle-Pock primal dual algorithm. The heuristic of iterated median filtering and weighted median filtering was incorporated to improve the flow accuracy.

Our key findings reveal that an efficient numerical scheme can be developed by incorporating the heuristic of iterated median filtering. However one must be careful in determining the number of iterations at each pyramid level. We observed that beyond 3 iterations the results were not so good visually. This was also visible in the AAE and EPE. The parameter $\eta$ associated with the additional constraint term must be small for better accuracy. Additionally, we found that the Chambolle-Pock algorithm provides an effective numerical scheme which can be adapted for a large class of non-smooth convex optimization problems.

Our work further goes on to show that classical formulations are still fairly competitive and produce good results when incorporated with modern optimization practices.

\section*{Acknowledgements}
The authors dedicate this paper to Bhagawan Sri Sathya Sai Baba, Revered Founder Chancellor, SSSIHL.

\end{document}